\documentclass{article} % For LaTeX2e
\usepackage{iclr2022_conference,times}

\usepackage{amsmath,amsfonts,bm}

\def\eqref#1{equation~\ref{#1}}
\def\1{\bm{1}}

\DeclareMathAlphabet{\mathsfit}{\encodingdefault}{\sfdefault}{m}{sl}
\SetMathAlphabet{\mathsfit}{bold}{\encodingdefault}{\sfdefault}{bx}{n}

\usepackage{hyperref}
\usepackage{url}
\usepackage{graphicx}

\usepackage{algorithm}
\usepackage{algpseudocode}
\usepackage{wrapfig}
\usepackage{amsthm}
\usepackage[font={small,it}]{caption}
\newtheorem{theorem}{Theorem}[section]

\title{SS-MAIL: Self-Supervised Multi-Agent \\Imitation Learning}

\author{
Akshay Dharmavaram$^1$, Tejus Gupta$^1$, Jiachen Li$^2$, Katia P. Sycara$^1$\\
$^1$Carnegie Mellon University, $^2$UC Berkeley\\

}

\iclrfinalcopy % Uncomment for camera-ready version, but NOT for submission.
\begin{document}

\maketitle

\begin{abstract}
The current landscape of multi-agent expert imitation is broadly dominated by two families of algorithms - Behavioral Cloning (BC) and Adversarial Imitation Learning (AIL). BC approaches suffer from compounding errors, as they
% use supervised learning to imitate the expert policy; however, this formulation 
ignore the sequential decision-making nature of the trajectory generation problem. Furthermore, they cannot effectively model multi-modal behaviors. 
% , thereby resulting in compounding errors and
% Furthermore, BC approaches also 
While AIL methods solve the issue of compounding errors and multi-modal policy training, they are plagued with instability in their training dynamics. 
In this work, we address this issue by introducing a novel self-supervised loss that encourages the discriminator to approximate a richer reward function. We employ our method to train a graph-based multi-agent actor-critic architecture that learns a centralized policy, conditioned on a learned latent interaction graph. 
We show that our method (SS-MAIL) outperforms prior state-of-the-art methods on real-world prediction tasks, as well as on custom-designed synthetic experiments.  We prove that SS-MAIL is part of the family of AIL methods by providing a theoretical connection to cost-regularized apprenticeship learning. Moreover, we leverage the self-supervised formulation to introduce a novel teacher forcing-based curriculum (Trajectory Forcing) that improves sample efficiency by progressively increasing the length of the generated trajectory. The SS-MAIL framework improves multi-agent imitation capabilities by stabilizing the policy training, improving the reward shaping capabilities, as well as providing the ability for modeling multi-modal trajectories.
\end{abstract}

\section{Introduction}
\begin{wrapfigure}{r}{0.4\linewidth}
\vspace{-14pt}
\includegraphics[width=\linewidth]{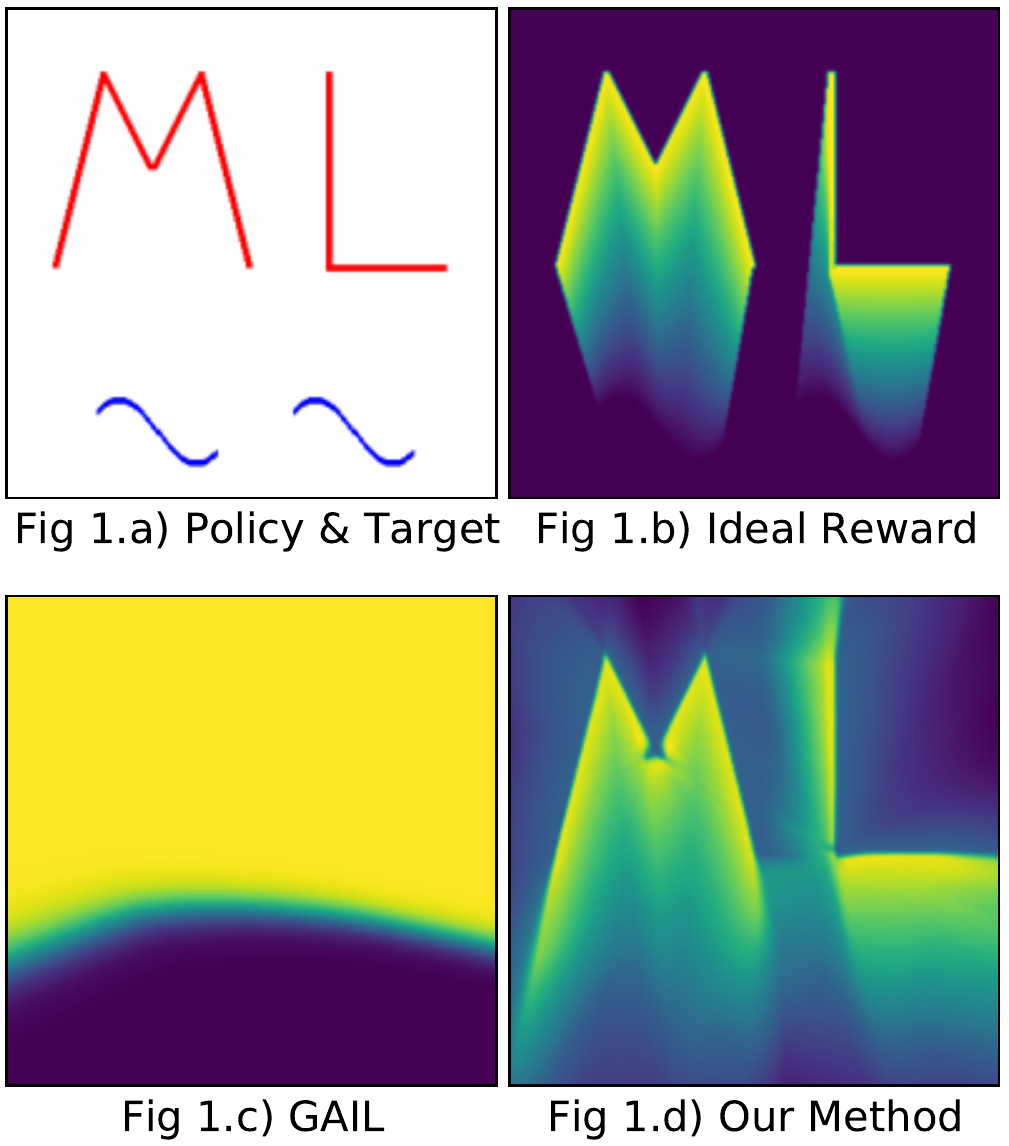}
\caption{(a) An example task that starts with an initial policy shown in blue, and aims to imitate the expert trajectories shown in red. (b)-(d) show the reward landscapes of the ideal reward function, the learned discriminator in GAIL, and our method (SS-MAIL), respectively.}
\vspace{-18pt}

\end{wrapfigure}

% % Imitation learning (IL) is a promising approach to learning intelligent behavior. 
% Adversarial Imitation learning methods, such as GAIL \citep{ho2016generative}, iteratively train a discriminator, and use it as a proxy reward function for updating a policy. In this work, we show that the the proxy reward function learned by GAIL fails to provide dense supervision for policy updates and leads to inefficient and unstable training. We suggest an improved algorithm SS-AIL that densifies this reward, and enables stable learning dynamics. We motivate our algorithm using a didactic example below.
Training an agent to imitate an expert is a promising approach to learning intelligent behavior and can be used in applications such as autonomous driving and robotic manipulation.
More specifically, the ability of the agent to robustly learn optimal policies in real-world scenarios is a current challenge facing the field. 
The most promising approaches for imitation learning are Behavioral Cloning (BC) and Adversarial Imitation Learning (AIL). BC methods have been shown to produce compounding errors \citep{ross2011reduction}, which makes it unsuitable for complex applications. 
Adversarial Imitation learning methods, such as GAIL \citep{ho2016generative}, iteratively train a discriminator, and use it as a proxy reward function for updating a policy. We show that the proxy reward function learned by GAIL fails to provide dense supervision for policy updates and leads to inefficient and unstable training.

Let's consider the example task of training two agents to draw the planar letters ``ML'' on a piece of paper, with the expert trajectories shown in red, and the initial policy shown in blue in Fig. 1(a). Training an RL agent to imitate the expert policy would ideally require a rich reward landscape with a clear gradient starting from the current policies, in the lower half of the image, and terminating at the expert sketches, as depicted in Fig. 1(b). However, the reward function learned by GAIL is almost constant throughout the state space, while abruptly changing at the decision boundary, as depicted in Fig. 1(c). In theory, GAIL would be able to learn the optimal policy with sufficient exploration; however, in practice, the training dynamics resulting from the sparse rewards in the local neighborhood of the current policy lead to sub-optimal policies. 
% Furthermore, the policy gradients in the vicinity of the expert trajectories are not sufficiently rich, as reward values in the neighborhood are all approximately the same value. 
We see that apart from a minuscule sliver of the state space, the agent is left in the dark when it comes to a prospective policy update.
% , thereby resulting in a sparse/vanishing gradient. There have been other attempts at remedying this sparse reward issue, such as WAIL; however, they are plagued with other issues such as unstable critic training, which we discuss in more detail later on. 
This leads us to wonder how the agents will ever learn to imitate the expert if they are left to explore in the dark. There have been attempts such as WAIL \citep{xiao2019wasserstein} to address this issue; however, it is also plagued with unstable critic-training, which we will elaborate upon in the upcoming sections. To briefly motivate the rest of the paper, we tease the results of our method (SS-MAIL) in Fig. 1(d), and we will explain in the upcoming sections the details of the framework that helped us get these results.

% The current SOTA GAIL framework attempts to learn a discriminator that approximates the reward function as shown in Fig 1.3. Adversarial methods such as GAIL make use of false positives by the discriminator to train a policy. However, a major assumption here is the assumption of the existence of false positives in the neighborhood of the current policy, as can be seen on the right in Fig 1.3. 

% In practise the policy explores locally and therefore won't reach the optimal policy as lack of signal in the neighborhood.
% We should transition from the thought process of learning a reward function, to learning a vector field, as the notion of gradients is more important than the absolute rewards. The states in the lower half of the image, where the current policy is located, has very sparse and negligible rewards in its neighborhood. 

The problem illustrated in the toy example above would be exacerbated in larger state spaces. We see that despite the recent progress in Adversarial Imitation Learning (AIL), the application of methods such as GAIL in scenarios with larger state spaces may be impeded by the exponential cost of exploring in ever-increasing hyper-spaces. The exploding cost of exploration would make tasks such as self-driving and robotic manipulation infeasible. These shortcomings are addressed in the SS-MAIL method introduced in this work.

\section{Background: Imitating Expert Multi-Agent Trajectories}

In this work, we consider sequential decision making problems, and we model them using the framework of Markov decision processes (MDP). An MDP can be represented as a tuple $\left(\mathcal{S}, \mathcal{A}, \mathcal{P}, r, \rho_{0}, T\right)$ with state-space $\mathcal{S}$, action-space $\mathcal{A}$, dynamics $\mathcal{P}: \mathcal{S} \times \mathcal{A} \times \mathcal{S} \rightarrow[0,1]$, reward function $r(s, a)$, initial state distribution $\rho_{0}$, and horizon $T .$ 

Imitation learning \citep{osa2018algorithmic} methods aim to learn task policies directly from expert demonstrations. The family of imitation learning algorithms is broadly divided into two classes, behavioral cloning (BC) \citep{bain1995framework} and Adversarial Imitation Learning (AIL) \citep{fu2017learning}. BC directly regresses from the expert's states to its decisions; however, such a straightforward supervised learning approach ignores the sequential nature of the problem and policy errors cascade during execution. 

Adversarial IL methods formulate the imitation learning problem as an adversarial game between the policy and the discriminator. The discriminator measures some divergence between the expert and policy's state-action distribution, and the policy aims to fool this discriminator. For instance, GaIL \citep{ho2016generative} minimizes the JS divergence between the expert and policy's state distribution. Their algorithm implements a min-max optimization procedure, 
\begin{equation}
\min _{\pi \in \Pi} \max _{D \in(0,1)^{S \times \mathcal{A}}}-\lambda_{H} H(\pi)+\mathbb{E}_{\pi}[\log (D(s, a)]+ \mathbb{E}_{\pi_{E}}[\log (1-D(s, a))]
\end{equation} 

AIL methods provide a compelling approach to imitation since they do not face the issue of compounding errors. However, such adversarial methods are known to be unstable to train. Our work aims to deal with this instability while retaining the advantages of AIL methods.

Previous work for multi-agent behavior prediction such as NRI \citep{kipf2018neural} and DNRI \citep{graber2020dynamic} model the expert's behavior using a VAE \citep{kingma2013auto} in which the encoder and the decoder networks are modeled as GNNs \citep{kipf2016semi} and the latent code represents the intrinsic interaction graph. Since their network architecture is designed to represent and work on the interaction graph of the agents, they can be used to understand the relationships between the agents qualitatively. These models are trained using multi-step BC and, therefore, cannot model multi-modal behavior. This is a significant drawback since real-world behavior is often multi-modal. Our work builds upon their policy architecture by training them using an actor-critic algorithm in an AIL setting.
% by changing the policy architecture to 

% \section{Method}
% This section describes the framework that is diagrammatically illustrated in Fig 2, above. We divide this section into three parts: \textbf{(1) SS-AIL:} where we discuss the specifics of our novel adversarial loss, \textbf{(2) MAIL:} where we discuss the specifics of our Graph Encoding Multi-Agent Actor-Critic Policy, \textbf{(3) Curriculum Generator:} where we discuss the specifics of our novel curriculum generator. 
\pagebreak
\section{Method}
In this section we will introduce SS-MAIL, and its components, namely: \textbf{1) SS-AIL: } a self-supervised AIL method to train the Discriminator, \textbf{2) MAIL: } an off-policy graph convolutional actor-critic architecture that imitates the expert using supervision from the Discriminator, \textbf{3) Trajectory Forcing: } a teacher forcing based curriculum generator for stabilized training dynamics. 
\subsection{SS-AIL: Self-Supervised Generative Adversarial Imitation Learning}
% We start by discussing the connection between Apprenticeship Learning and our novel SSGAIL framework. 
% Next, we introduce the training framework, and finally discuss a few extensions of the self-supervised loss.

\subsubsection{From Apprenticeship Learning to SS-AIL}
\citet{ho2016generative} show that GAIL is a form of apprenticeship learning with a specific cost regularization, which optimizes the min-max objective: 
\begin{align}
\psi_{\mathrm{GA}}^{*}\left(\rho_{\pi}-\rho_{\pi_{E}}\right)=\max _{D \in(0,1)^{\mathcal{S}} \times \mathcal{A}} \mathbb{E}_{\pi}[\log (D(s, a))]+\mathbb{E}_{\pi_{E}}[\log (1-D(s, a))]
\end{align}
In practise, the rewards from the discriminator do not provide rich policy-gradients, which result in sub-optimal policies. We address this issue by introducing a self-supervised loss function for our discriminator that encourages it to provide an enriched reward signal, not only in the local neighborhood of the current trajectory, but globally as well. 

Our self-supervised loss is presented in Eq. 3, below. We sample $\alpha \in [-1,1]$ and take the weighted average of the trajectories resulting from the policy and the expert, specifically $\tau_{G} \sim \pi$ and $ \tau_{E} \sim \pi_{E}$. This new trajectory can be mathematically formulated as $\tau_{\alpha} = \alpha \tau_{G} + (1-\alpha) \tau_{E}$. Our intuition behind creating the self-supervised loss is that the state action pairs of $\tau_\alpha$ should regress smoothly based on their relative distance to $\tau_G$, or in other words $\mathbb{E}_{\tau_{\alpha}}[D_\theta (s, a)] = \alpha$. Thus, we convert the Binary Cross Entropy loss to a Mean Squared Error (MSE) loss and use the sampled $\alpha$ values as self-supervised labels.
\begin{align}
\psi_{\mathrm{SS}}^{*}\left(\rho_{\pi}-\rho_{\pi_{E}}\right)= \mathbb{E}_{\pi, \pi_E, \alpha}\left[ 
\underbrace{\left(0-D(s_G,a_G)\right)^2}_{\text{Generated MSE}} + 
\underbrace{\left(1 - D(s_E,a_E)\right)^2}_{\text{Expert MSE}} + \right.
\left. \underbrace{\left(\alpha - D(s_\alpha,a_\alpha)\right)^2}_{\text{Self-Supervised MSE}}  
\right] 
\end{align}

An interesting consequence of this formulation is that the third term in the self-supervised optimization objective, the self-supervised MSE term, serves as a natural from of exploration during the policy optimization stage.

\begin{theorem} SS-AIL is an instantiantion of cost-regularized apprenticeship learning, i.e. the policy at the saddle point $(\pi, D)$ of the min-max problem described above is a solution to $\min_\pi \max_c ( -\psi(c) + E_{\pi} [c(s,a)] - E_{\pi_E}[c(s,a)])$ for some specific $\psi$.
\end{theorem}
\begin{proof}
Refer to the appendix.
\end{proof}

\begin{algorithm}[H]
\caption{SS-AIL}\label{alg:cap}
\begin{algorithmic}
\State $\textbf{Input: } \text{Expert Trajectories } \tau_{E} \sim \pi_{E}, \text{Initial Policy } \pi_{\phi}, \text{Initial Discriminator } D_{\theta}$
\State $\textbf{Initialize: } \text{Policy } \pi_{\phi}, \text{Discriminator } D_{\theta}$
% \Ensure $y = x^n$
% \State $y \gets 1$
% \State $X \gets x$
% \State $N \gets n$
\While{Policy Improves}
    \State $\text{Sample Trajectories } \tau_{G} \sim \pi_{\phi}$
    \State $\text{Sample } \alpha \text{ and Compute: } \tau_{\alpha} = \alpha \tau_{G} + (1-\alpha) \tau_{E}$
    \State $\text{Update } D_{\theta} \text{ using gradient:}$
    \begin{equation}
        \mathbb{E}_{\tau_{G}}\left[\nabla_{\theta}  \left(D_{\theta}(s, a)\right)^2\right]+\mathbb{E}_{\tau_{E}}\left[\nabla_{\theta} \left(1-D_{\theta}(s, a)\right)^2\right] + \mathbb{E}_{\tau_{\alpha}, \alpha}\left[\nabla_{\theta} \left(\alpha-D_{\theta}(s, a)\right)^2\right]   \nonumber
    \end{equation}
    \State $\text {Update } \phi \text { with SAC to increase the following objective: } \mathbb{E}_{\tau_{i} \sim \{ \tau_{G},\tau_{\alpha}\}}\left[D_{\theta}(s, a)\right]$
\EndWhile
\end{algorithmic}
\end{algorithm}
\subsubsection{Algorithm}
In the SS-AIL algorithm we start by sampling a trajectory from the current policy. Next, we compute the weighted average between the current trajectory and a sampled expert trajectory. We denote this trajectory as $\tau_\alpha$. According to our Discriminator update rule, derived previously, we train our discriminator to learn a smooth reward function by reconstructing the self-supervised labels, $\alpha$. Next, we need to update our policy function. 

In GAIL, \citet{ho2016generative} use the on-policy TRPO update rule. However, according to the min-max function derived above, our policy should be optimized on not only the current trajectory, but also $\tau_\alpha$. This serves as an inherent exploration term for our policy training, as we are exploring states that are disparate from the current trajectory. Thus, we use the SAC update policies to accommodate the off-policy updates from $\tau_\alpha$.

We formulate our update rules such that the discriminator converges faster than the actor, similar to the two-timescale approach in the actor-critic setting. Thus, we no longer require balancing the discriminator, as seen in GAIL. Also, the change in our Discriminator values converges to zero, as opposed to WAIL, which constantly updates the fluid surface of the Discriminator output, as there is no grounding of the outputs to any specific value.

\subsubsection{Self-Supervision as a form of Reward-Shaping}
\begin{wrapfigure}{r}{0.25\linewidth}
\vspace{-10pt}
\includegraphics[width=\linewidth]{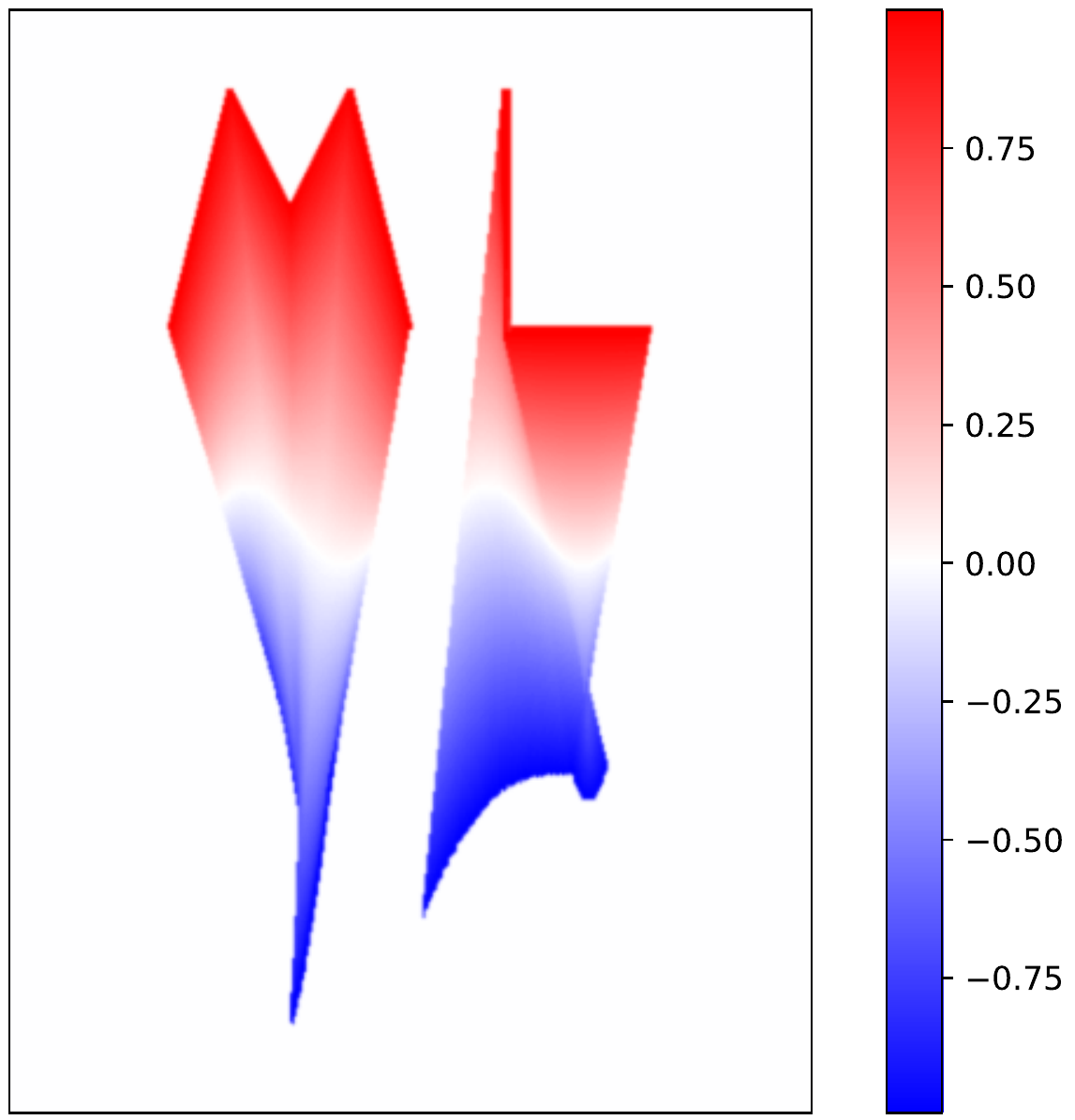}
\caption{Vector field induced by negative $\alpha$.}
\end{wrapfigure}
Our novel self-supervised loss gives us considerable flexibility in the sampling approaches used to obtain $\alpha$. 
As stated previously in our discussion on the discriminator approximating a vector field, we noticed that only taking $\alpha$ values from [0,1] would be insufficient for training as during exploration, the agent would explore states that would naturally correspond to negative $\alpha$ values. As the discriminator was not trained on these states, generalization isn't guaranteed. Thus, to complete the neighborhood around the current policy trajectory, we also consider negative $\alpha$ values. We show an empirical analysis in the next section. The negative $\alpha$ values assist in inducing a complete vector field around the neighborhood of our trajectory, as shown in Fig 3.

We can custom design the sampling approach and the loss functions to suit our need, which will open up greater flexibility and control for the community to experiment with. We elaborate in the Appendix.

\subsection{MAIL: Multi-Agent Imitation Learning}
In this section we introduce the architecture of the policy network, which is divided into two parts - the graph encoder and the graph soft actor-critic (G-SAC). In our architectures, we use the message passing standard established by \citet{kipf2018neural}, and model the node-to-edge and edge-to-node message passing operations as follows:
\begin{align}
\mathbf{h}_{(i, j)}^{l} &=f_{e}^{l}\left(\left[\mathbf{h}_{i}^{l}, \mathbf{h}_{j}^{l}, \mathbf{s}_{(i, j)}\right]\right); &
\mathbf{h}_{j}^{l+1} =f_{v}^{l}\left(\left[\sum_{i \in \mathcal{N}_{j}} \mathbf{h}_{(i, j)}^{l}, \mathbf{s}_{j}\right]\right)
\end{align}

where the subscripts of $(i,j)$ represent edge related features, while subscripts denoted with single letters, such as $i$ and $j$ represent node related features. $h$ represents the corresponding edge or node embedding, and $s$ represents states. Finally, $f$ represents a neural network function approximator, and $\mathcal{N}_{j}$  corresponds to indices of the neighboring nodes connected to the node indexed by the subscript, which in this example is $j$.

\subsubsection{Graph Encoder}
The goal of the Graph Encoder is to infer the underlying interaction graph of the various agents, conditioned upon the observed history. This latent interaction graph will be used as weights in the graph convolution step in G-SAC. Our Graph Encoder is based on the encoder network \citet{graber2020dynamic}. We employ a fully connected graph convolutional network to output a distribution over edge weights. However, unlike in DNRI, we do not include prior networks that train for precognition of the future evolution of the graph based on the observed history.  Our model takes in as input the cumulative observable state space, and employs an LSTM to keep track of the history. We model this as follows:
\begin{align}
\mathbf{h}_{(i, j), \mathrm{enc}}^{t+1} &=\mathrm{LSTM}_{\mathrm{enc}}\left(\mathbf{h}_{(i, j)}^{t}, \mathbf{h}_{(i, j), \mathrm{enc}}^{t}\right), \\
\underbrace{q_{\phi}\left(\mathbf{z}_{(i, j)}^{t} \mid \mathbf{x}\right)}_{\text{Graph Encoder Output}}=& \operatorname{softmax}\left(f_{\mathrm{enc}}\left(\mathbf{h}_{(i, j\rangle, \mathrm{enc}}^{t}\right)\right)
\end{align}

\subsubsection{Graph Soft Actor-Critic (G-SAC)}
G-SAC samples a dynamically computed interaction graph from the Graph Encoder and uses the graph as the activations for the respective edges among the nodes of the graphs. It then computes a graph convolutional to obtain the mean and standard deviation values, similar to the output of SAC \citet{haarnoja2018soft}. We also include another head which functions as a critic. The critic head approximates the Q-function of the policy and is used to train the network. As in SAC, we also keep a target critic, which is updated based on polyak averaging.
\begin{align*}
\mathbf{h}_{(i, j\rangle}^{t} &=\sum_{k} z_{i j, k} f_{e}^{k}\left(\left[\mathbf{x}_{i}^{t}, \mathbf{x}_{j}^{t}\right]\right) \\
\boldsymbol{\mu}_{j}^{t+1}=f_{\mu}\left(\sum_{i \neq j} \mathbf{h}_{(i, j)}^{t}\right); \qquad & \qquad      \boldsymbol{\sigma}_{j}^{t+1}=f_{\sigma}\left(\sum_{i \neq j} \mathbf{h}_{(i, j)}^{t}\right);\\
\underbrace{p\left(\mu_{j}^{t+1}, \sigma_{j}^{t+1} \mid \mathbf{x}^{t}, \mathbf{z}\right)}_{\text{G-SAC Output}}&=\mathcal{N}\left(\boldsymbol{\mu}_{j}^{t+1}, \sigma^{2} \mathbf{I}\right)
\end{align*}

\subsection{Trajectory Forcing}
\begin{wrapfigure}{r}{0.4\linewidth}
\vspace{-20pt}
\includegraphics[width=\linewidth]{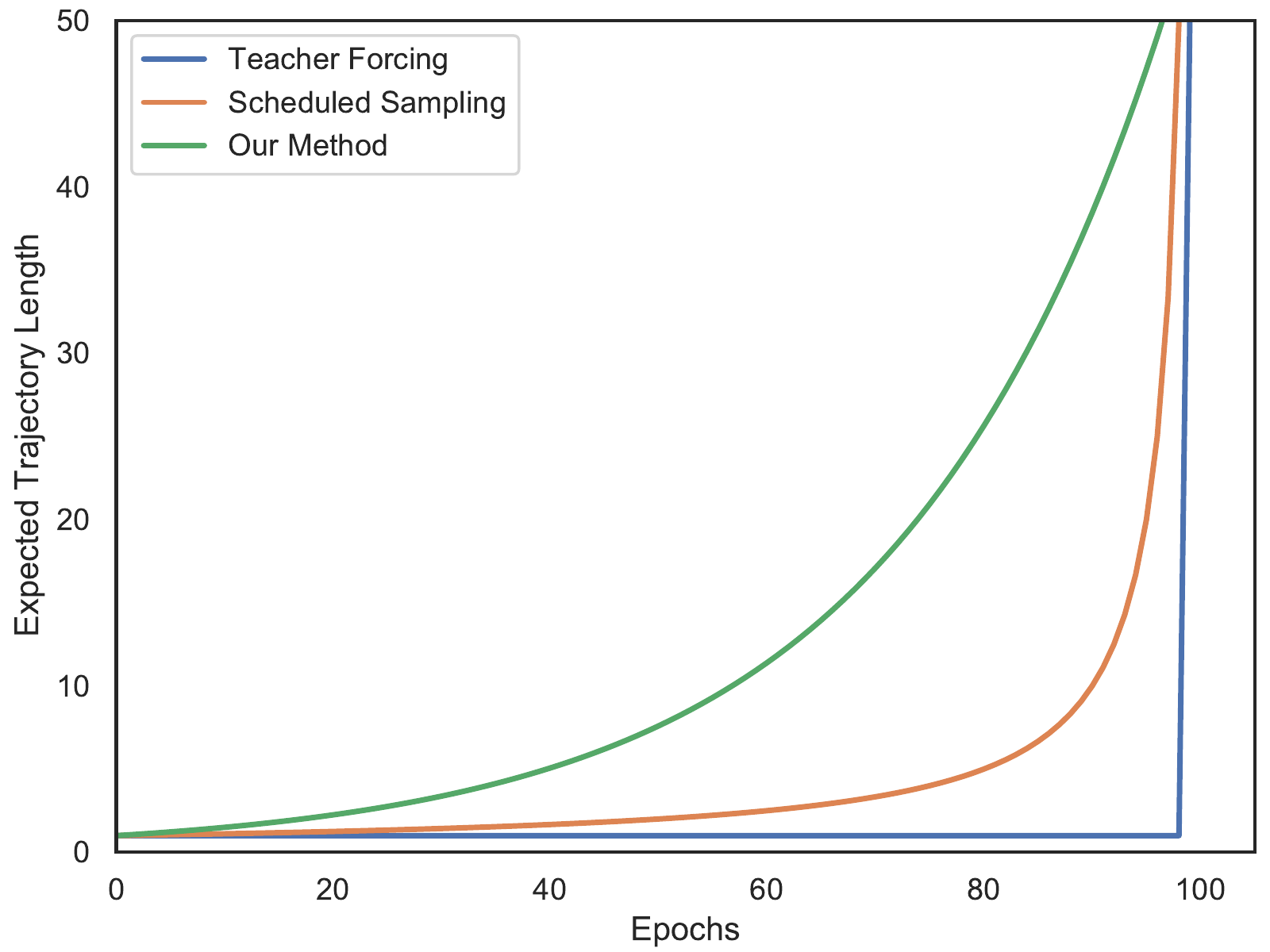}
\caption{We plot the expected length of the generated trajectory during training in between pedagogical intervention for teacher forcing, scheduled sampling and trajectory forcing. We show that our curriculum provides a gradual increase in expected trajectory lengths for better policy training.}
\end{wrapfigure}
Teacher Forcing based curriculums are an important tool for training sequence models \citep{ranzato2015sequence}. In a sequence generation setting, we expect the learned models to use the outputs of the previous timesteps as inputs. However, training such generative models was shown to be sensitive to weight intitialization \citep{yu2017seqgan}, in the absence of pedagogical intervention, such as teacher forcing. Scheduled Sampling \citep{bengio2015scheduled} is a curriculum based on teacher forcing that attempts to gradually transition from using the ground-truths as inputs, to using the model's previous outputs. NRI and DNRI use teacher forcing in their approaches; however, such an approach could result in compounding errors during test time, due to the distributional shift. The use of such curriculums is not prevalent in AIL, as the frameworks accommodate only for on-policy updates. There are methods that try to work around this by pre-training using BC \citep{jena2020augmenting}. However, in SS-AIL, the off-policy functionality is built into the formulation. Thus, we can leverage $\alpha$ values close to 1 and use them as a proxy for teacher forcing. Our goal is to take advantage of the exploration functionality of our loss function to gradually reduce the teacher forcing frequency to zero. In other words, our curriculum progressively increases the intervals of pedagogical intervention. We show this in Section 4.5. 

We mathematically model the frequency of interventions as $1.5^{-\text{epoch}/\beta}$, where beta is a hyperparameter that assists in generating a progressively increasing sequence.
% This method is inspired by how toddlers are taught to ride a bike. The basic idea behind our approach is simple: we start by sampling the state solely from the expert trajectory, and feeding them into our generative model. 
Our intuition behind modeling the length of the generated trajectory as an exponential with respect to the epoch, is that we intend on doubling the size of the generated trajectory every $\beta$ epochs. A linear model would result in reduced sample efficiency during the later stages of training, as the  generalizability of the model would outpace the curriculum. The exponential increase ensures efficient data utilization.
\pagebreak

\section{Results}
\label{experi}
% I think we should talk about how we are so invested in beating the baselines, that we don't stop and think about how the algorithms generalize and what they learn. Refer to that VLR paper. You can also talk a bit about how a random discriminator yields a good policy in mujoco.
In this section, we quantitatively evaluate the strengths of our framework using two custom-built environments and a real-world dataset. These experiments illustrate the advantages of the SS-MAIL framework, namely: i) increased stability of its training dynamics, ii) multi-modal trajectory modeling capabilities, iii) increased sample-efficiency of Trajectory Forcing, iv) enhanced versatility in reward-shaping using self-supervision, and v) robustness to compounding errors on real-world datasets.
% ,    We compare them to prior work, and reason about the design choices that were made. 
% By the end of this section, we intend on illustrating the :
% \begin{enumerate}
% \item The robustness of training dynamics.
% \item  of MAIL.
% \item Effectiveness of curriculum training.
% % \item Qualitative analysis learned interaction graphs.
% \item Performance of MAIL in a control setting.
% \item Performance of MAIL on a real world prediction task
% \item Contribution of negative alphas for enhanced training dynamics
% \end{enumerate}

% We introduce themajor sections as questions 
% \textbf{Question 1.} asd \texttt{What if your trajectories are multi-modal?} 
% hello
\subsection{Experimental Setup}
To evaluate the capabilities of our methods against prior SOTA baselines, we evaluate our framework on one custom-built environment (Y-Junction), and one real world prediction task (Noisy Mocap). To minimize the prevalence of confounding variables, the experiment design of Y-Junction was simplified considerably to ensure that any promising improvement in performance is warranted by the inherent attributes of the frameworks, and not any other external confounding factors, such as variable overflow \citep{dietz2015understanding}. The Y-Junction environment is a trajectory forecasting environment that tests the multi-modal capabilities of AIL and BC based frameworks by simulating a 3-way Y-Junction scenario. The environment has three agents that follow one another on a one-way street until the beginning of a fork or Y-Junction. Then the lane splits into two, thus simulating two potential modes. Our goal is to pick one of the two lanes or modes. We provide more details in the appendix. The Noisy Mocap prediction task involves training a multi-agent policy on the CMU Motion Capture Dataset \citep{mocap} for subject \#35, and observing the zero-shot generalization in the presence of noisy inputs.

\subsection{Stability in Training Dynamics}
\begin{wrapfigure}{r}{0.4\linewidth}
\vspace{-30pt}
\includegraphics[width=\linewidth]{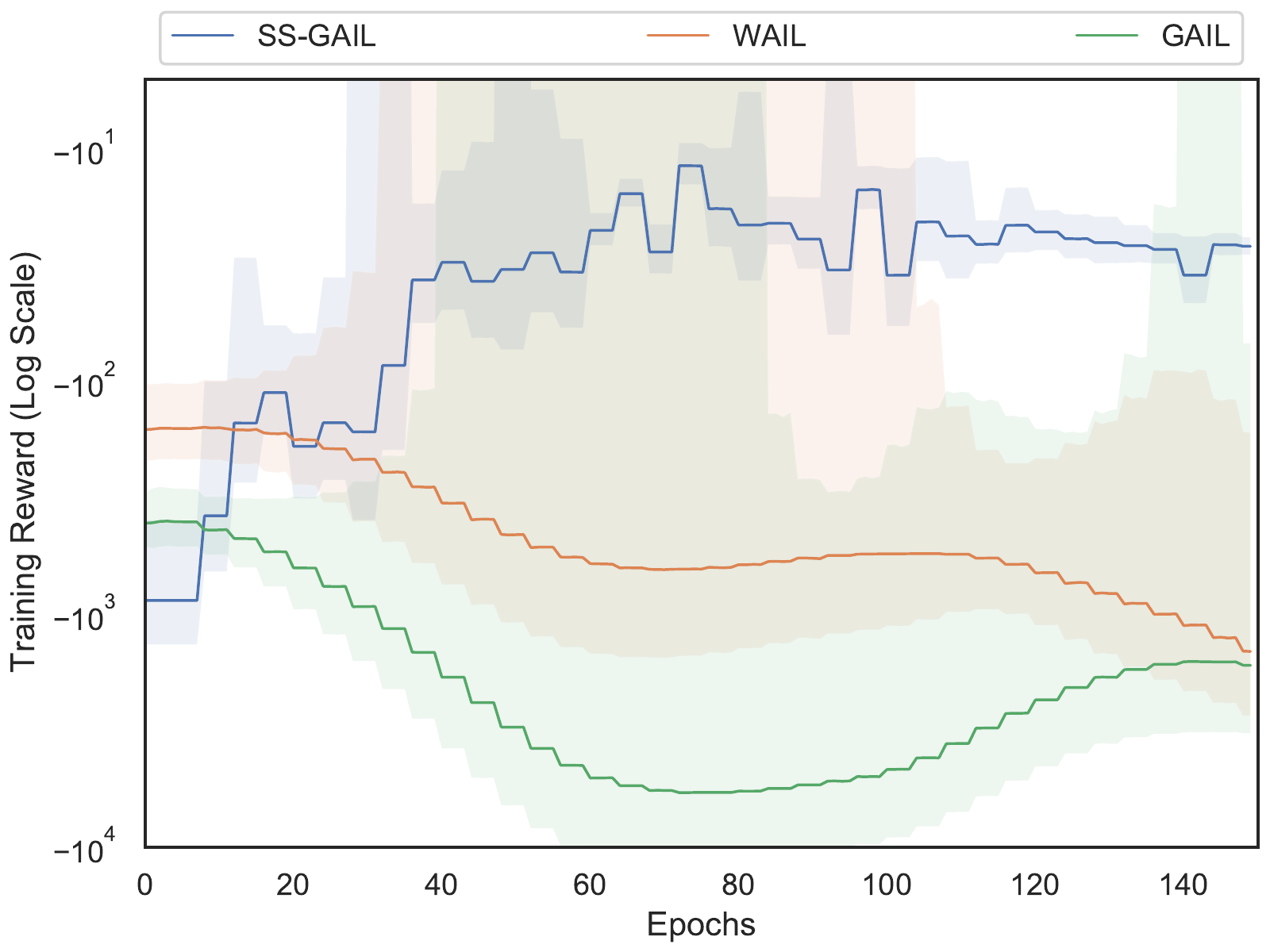}
\caption{Training error over time in the Y-Junction environment shows that SS-AIL successfully imitates the expert. The low standard deviation of error for SS-AIL despite unfavorable initialization demonstrates the increased robustness of its training dynamics.}
\vspace{-10pt}
\end{wrapfigure}
Robustness of the training algorithm to initialization, stochasticity in the environment, and training data is very important to assure predictability in the training dynamics. The Y Junction experiment, upon which we evaluate the robustness, provides a simple multi-agent testing scenario to illustrate the differences between various AIL approaches. To have a uniform evaluating strategy, we use the same architecture for all methods, and solely swap-out the corresponding AIL loss functions. We elaborate further in the Appendix.

In Fig. 4, we observe that the training loss of SS-AIL reduces to zero, implying that the algorithm is able to successfully imitate the expert during training. SS-AIL quickly converges  to the expert policy for all initializations; however, GAIL and WAIL are unable to converge despite good policy initializations. We see that GAIL and WAIL start-out with better performances, due to their policy initializations; however, are unable to consistently improve and fluctuate considerably. As the Y-Junction experiment is unbounded in its state-space, there is no boundary that constrains diverging policies, thus resulting in the observed exploding losses. In contrast, we see that the performance of the policy trained with SS-AIL consistently improves, and converges to zero. These results can be attributed to the richer family of reward functions being approximated by the Discriminator. The Binary Cross Entropy loss of GAIL \citep{ho2016generative} results in sparse reward signals for the policy training, thus resulting in unstable training. Furthermore, the inablility of the WAIL Discriminator to converge upon a designated output surface results in a non-stationary value-function, thereby destabilizing the training of the critic. This instability becomes more pronounced as the policy and expert approach each other in the space of trajectories, as small variations in the surface of the discriminator output would lead to prominent shifts in the value-function. 
% To evaluate the robustness of SS-MAIL, compared to GAIL and WAIL, we use the same architecture as SS-MAIL, and replace the SS-AIL training algorithm with GAIL and WAIL. We evaluate the performance on Y-Junction. In Fig 5, we observe that SS-AIL outperforms GAIL and WAIL. However, upon a cursory glance at Fig 5, one would notice that SS-AIL initially starts out with a worse loss than both GAIL and WAIL, and quickly overtakes both algorithms. The reasoning behind this can be attributed to Survivorship Bias. GAIL and WAIL have comparatively unstable training dynamics. Thus, initialization and hyper-parameter tuning play an important role in training. For most configurations of GAIL and WAIL, the initial loss is lower than -500, similar to the SS-AIL curve. However, these runs were unable to learn meaningful policies and were appropriately discarded. Thus, after considerable hyper-parameter tuning we obtained the plots shown in Fig 5. As expected, WAIL outperforms GAIl. However, we notice that WAIL still underperforms SS-AIL. The reasoning behind this can be attributed to the inability for the critic to catch up to the discriminator. As the WAIL loss does not anchor the outputs of the discriminator to any specific values, the output of the discriminator is fluid, which means it does not remain constant across epochs. This makes it considerably difficult for the critic to approximate the Q-function.

% We also tried variations of GAIL and WAIL, without any appreciable success. We elaborate upon these experiments in the Appendix.

\pagebreak

\subsection{Multi-Modal Capabilities}
\begin{wrapfigure}{r}{0.4\linewidth}
\vspace{-30pt}
\includegraphics[width=\linewidth]{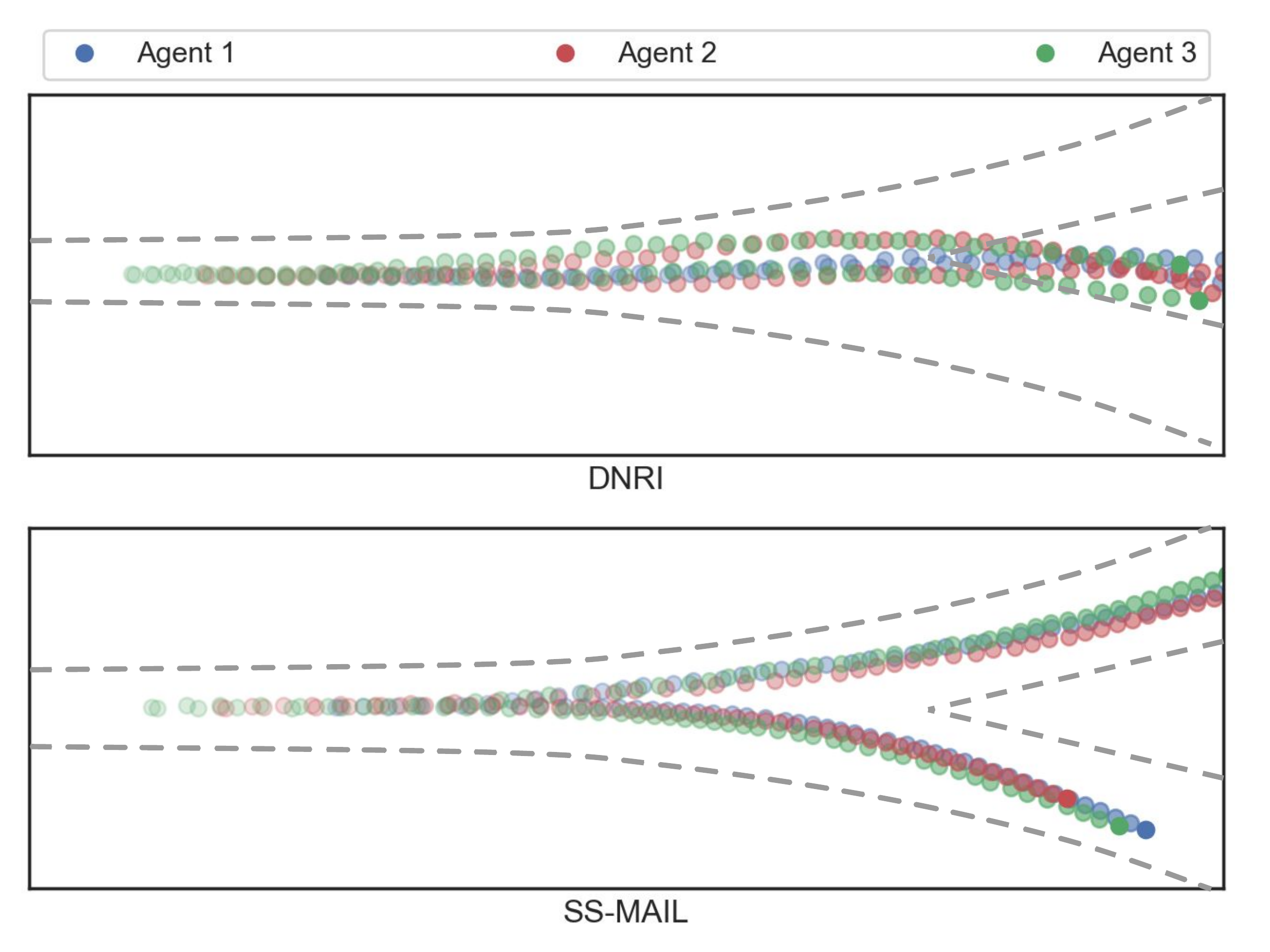}
\caption{A visualization of the multi-modal trajectories learned by SS-AIL and DNRI on the Y-Junction task. We observe that DNRI averages over the different modes, while SS-AIL successfully differentiates between them. The inability to distinguish between multi-modal expert trajectories can prove disastrous in continuous state settings, as this may lead to visiting states that differ considerably from the expert state-distribution.}
\vspace{-10pt}
\end{wrapfigure}
Differentiating between the various modes in the expert-trajectory dataset is essential for training multi-modal policies. The gradient of the policy training loss should be in the direction of a specific mode, instead of the weighted average over all the modes. The Y Junction example provides a simple multi-agent testing scenario that demonstrates how BC based methods learn sub-optimal policies in the presence of multiple modes. To evaluate SS-AIL and DNRI on their multi-modal modeling capabilities, we ensure both algorithms have similar policy architectures and differ solely in their BC and AIL based policy training steps.

We plot the trajectories in Fig 5, and illustrate issue of modal averaging, observed in BC methods. We observe that SS-AIL converges to the different expert modes, implying that the Discriminator successfully approximates a reward function that has a discernible preference among prospective modalities. We do not observe this property in BC based DNRI algorithm, as shown in Fig. 5. The self-supervised loss function has an inherent positive feedback loop that progressively increases the gradients of the rewards for imitating the closest expert modality. This property is a direct consequence of the self-supervised loss, as the slope of the reward is inversely proportional to the distance between the policy and expert trajectories.

\subsection{Sample Efficiency and Curriculums}
\begin{wrapfigure}{r}{0.4\linewidth}
\vspace{-25pt}
\includegraphics[width=\linewidth]{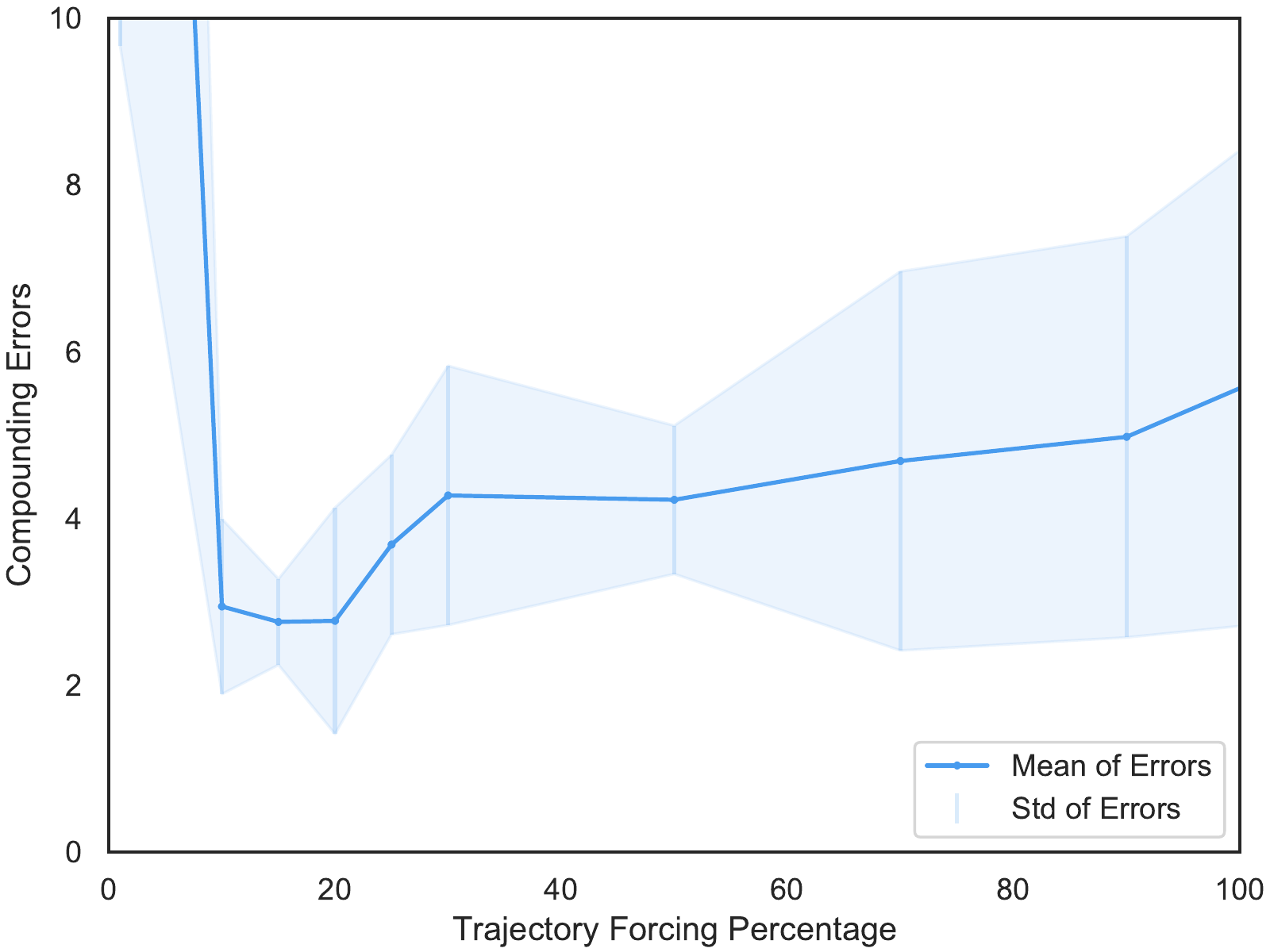}
\caption{Running an ablation, over $\beta$, for the mean and standard deviation of testing error in the Y-Junction environment illustrates the existence of $\beta$ values that result in both low mean and standard deviation.  This highlights Trajectory Forcing's ability to improve the robustness to unfavorable weight initializations, increase the sample efficiency, as well as eliminate the issue of compounding errors.}
\end{wrapfigure}
Teacher Forcing alleviates the issues of unfavorable initializations and low sample efficiency during the initial stages of training; however, it results in compounding errors during testing. The Y Junction experiment is a simple multi-agent scenario that demonstrates how Trajectory Forcing alleviates unfavorable initializations, increases sample efficiency during the initial stages of training, and eliminates the issue of compounding errors. To evaluate the effectiveness of the Trajectory Forcing curriculum, we train SS-AIL with different values of $\beta$ to show an ablation over $\beta$, which ranges from no Teacher Forcing (0\%), seen in AIL approaches, to complete Teacher Forcing (100\%), seen in BC approaches. In Figure 6, we focus on two aspects of the plot, the mean and standard deviation. A smaller mean loss, averaged over multiple random seeds, implies robustness in policy training to different model initializations. Moreover, a low standard deviation implies superior generalization during testing, and thereby lower compounding errors. we see that without teacher forcing the mean and standard deviation of the testing loss is considerably high, which is common for AIL methods. Teacher Forcing reduces the mean loss considerably; however, the standard deviation is high due to the compounding errors.
We observe that for a $\beta$ value of around 15\%, the trajectory forcing curriculum considerably reduces the mean loss as well as the standard deviation during testing.
This demonstrates the robustness to weight initializations, improvement in sample efficiency, and reduction in compounding errors.

% This can be understood as intervening at intervals that gradually increase as training progresses, starting from 100\% pedagogical intervention and culminating in no pedagogical intervention. Thus, we take advantage of Teacher Forcing's ability to stabilize the initial stages of training, while also gradually removing it it to eliminate the issue of compounding errors. Furthermore, the smooth increase in the mean trajectory length further stabilizes training, compared to scheduled sampling, which we expand upon further in the Appendix.

% \subsection{Learned Interaction Graphs}
% end at 7.5 (0.5 pgs)

% \subsection{Performance in a control setting (\textbf{Not done})}

% This is an important experiment as we may not always have a differentiable simulator to back-propogate through,

% Here, we will give the dynamics to DNRI, and see if SS-MAIL can match the performance.

% The main issue with control tasks is that we can't backpropogate through the environment. 

% We construct a similar multi-modal experiment that involves controlling a set of motors.

% For the sake of comparing against prediction algorithms, such as DNRI, we provide the prediction baselines with the environment model. We do not provide the environment model to MAIL. Our goal is to see if our control algorithm can perform just as well as the prediction algorithms with the environment model.

% We want our method to be able to compare favorably against the prediction algorithms despite not having the model.
\pagebreak

\subsection{Zero-Shot Generalization on Noisy Real-World Data}
\begin{wrapfigure}{r}{0.4\linewidth}
\vspace{-10pt}
\includegraphics[width=\linewidth]{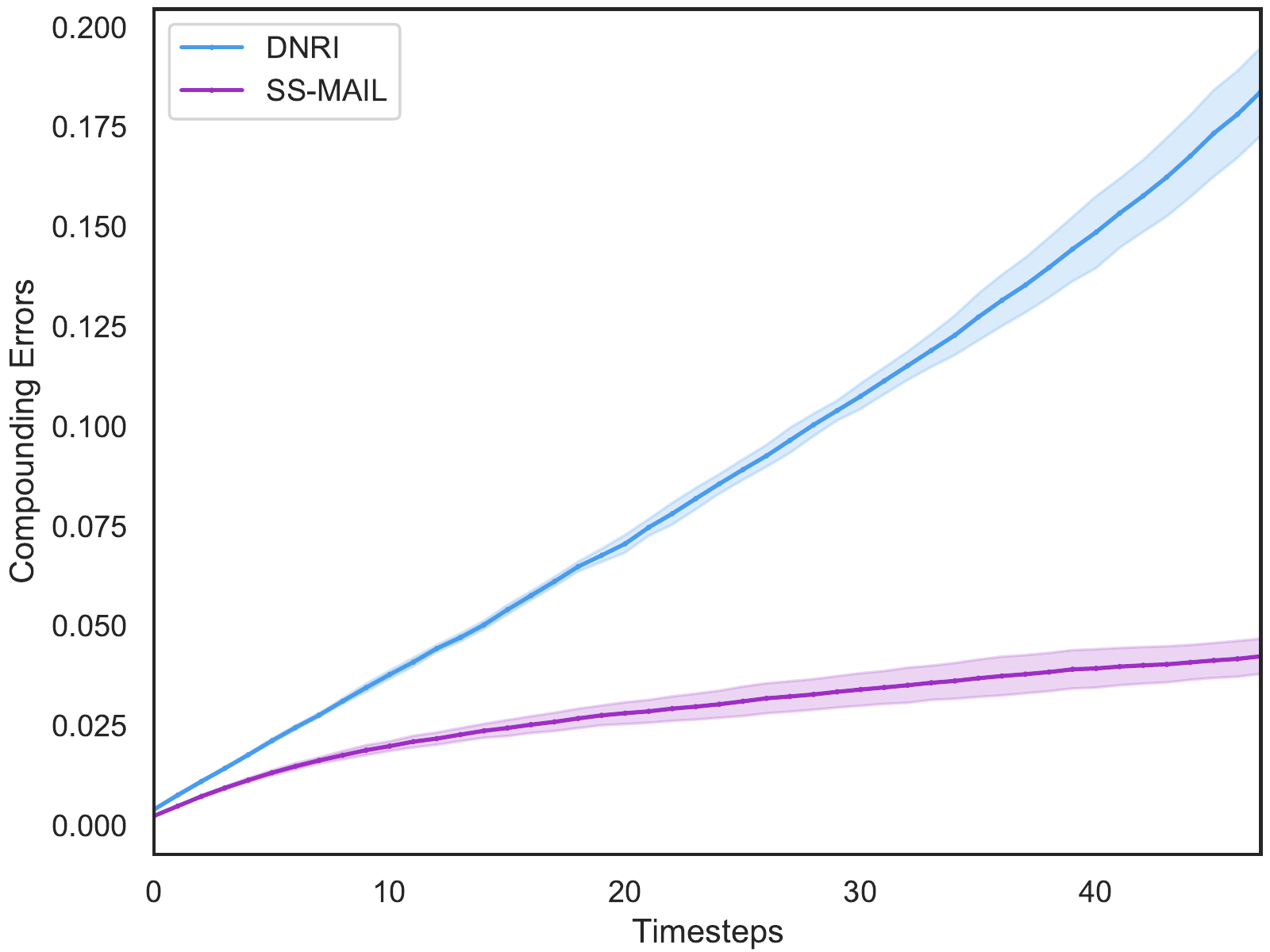}
\caption{We plot the Compounding Errors over time for the Noisy Mocap environment shows that SS-AIL  successfully  compensates for noisy inputs during test time, zero-shot. The low standard deviation and the decrease in slope for SS-AIL illustrate its generalization capabilities.}
\vspace{-10pt}
\end{wrapfigure}
Compounding errors occur when the model is unable to correct for deviations from the training distribution of states, and fails to compensate for the observed deviations. This can be attributed to poor generalization, and is commonly caused by noise.
The Noisy Mocap experiment is a real-world multi-agent prediction dataset that demonstrates how negligible amounts of Gaussian noise, with standard deviation of 0.05, can accumulate and permanently derail the generation of trajectories.
To evaluate the issue of compounding errors in the presence of noise during test time, we evaluate SS-AIL and DNRI using similar policy architectures. 

In this experiment, we observe that DNRI progressively deviates further from the expert trajectories as the trajectory length increases. We also observe a positive correlation  between the standard deviation and the trajectory length. However, in the case of MAIL we do not observe a similar linear increase in the compounding errors. Moreover, the slope decreases with the increase in trajectory length.  This observed deviation can be attributed to the loss functions used to train the respective algorithms. BC approaches are trained to precisely match the expert. Despite having a lower validation loss, generalizing zero-shot to noisy environments is not guaranteed, as seen in Fig 7. SS-AIL, on the other hand, learns a policy that maximizes the cumulative discriminator return. Thus, the trained policy is akin to a vector field surrounding the expert state-distribution, as it learns to optimize for reaching rewarding states, thereby resulting in robust recovery even in noisy environments.

\subsection{Reward Shaping using Self-Supervision}
\begin{wrapfigure}{r}{0.4\linewidth}
\vspace{-20pt}
\includegraphics[width=\linewidth]{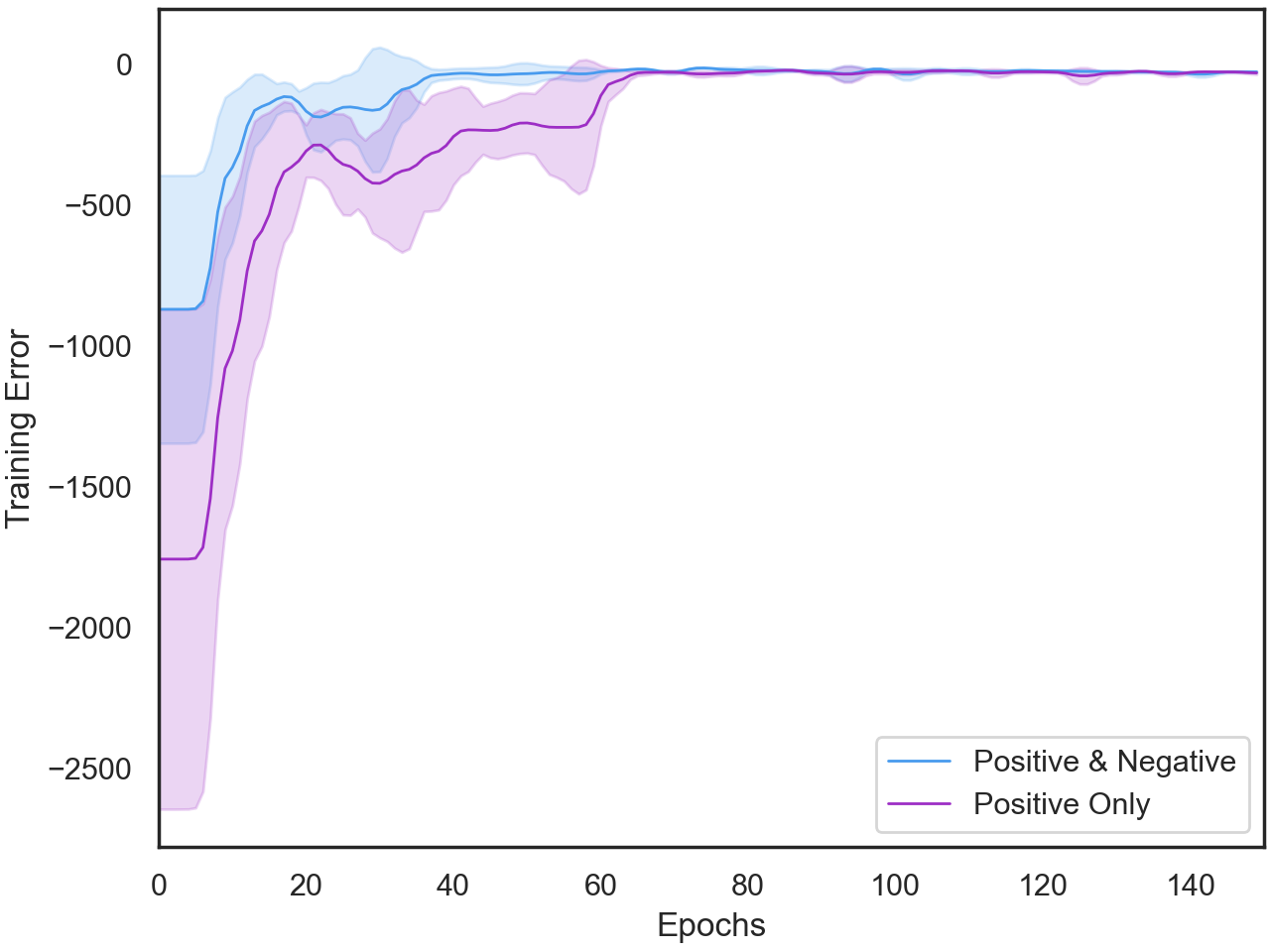}
\caption{We plot the training loss over epochs for the Noisy Mocap environment and show that sampling negative $\alpha$ values improves the final loss and speeds up training. This can be attributed to the richer reward gradient in the local neighborhood of the current trajectory.}
\vspace{-10pt}
\end{wrapfigure}
Reward-Shaping provides engineers and scientists with a way to meticulously control the reward function, resulting in more efficient policy training and reduction in training time. In our method, the self-supervised sampling allows for precise control over the reward function the discriminator is approximating. In this experiment, we investigate one such customization of the reward function that results in enhanced training dynamics. The Noisy Mocap experiment provides us with a real-world dataset to investigate the ramifications of different sampling procedures on the training dynamics of SS-AIL.
To evaluate the effect of the self-supervised sampling, we keep the overall SS-AIL algorithm constant and solely change the sampling distributions of $\alpha$.

We see that sampling only positive values of $\alpha$ results in training dynamics that take longer to converge. We observe that with the introduction of negative $\alpha$ values, the training dynamics substantially improves. 
This demonstrates the importance of reward shaping using self-supervision, which opens up the possibilities for researchers to design more intricate reward function approximators.
The improved performance observed upon the introduction of negative $\alpha$ values can be attributed to the resulting completion of the reward function in the local neighborhood of the current trajectory. As the current trajectory corresponds to an $\alpha$ value of zero, then the absence of negative $\alpha$ values during self-supervised sampling would result in an incomplete reward function in the local neighborhood of the policy's state-distribution. 

The inclusion of negative $\alpha$ values is one such example of reward-shaping for enhanced performance and stable training dynamics. Based on the application, engineers and scientists can design custom rewards for the enhancing the policy training, based on the application constraints.

\section{Related Work}
In this section we provide a brief overview of the various research contributions that share similarities to the approaches described in our work. We start by discussing the recent work on graph-structured data, and then delve into multi-agent trajectory forecasting, and finally, imitation learning.

Graph structured data does not possess a fixed structure, and thus can not be used with traditional neural networks. To address these issues, there have been approaches such as \citep{li2015gated}, \citep{scarselli2008graph}, and \citep{kipf2016semi} that introduce the notion of message passing to accommodate for the irregular, yet structured nature of graphical data. In this paper, we use  Graph Convolutional Networks \citep{kipf2016semi} for our internal message passing; however, there are other methods proposed, such as Graph Attention Networks \citep{velivckovic2017graph}, and Graph Recurrent Networks (GRN) \citep{hajiramezanali2019variational} as prospective forms of message passing. GRN also provides the notion of recurrence, which can be used to model time-series data. The Graph Convolutional RL \citep{jiang2018graph} work is a relevant bridge between graph structured data and the sequential decision making setting, which employs message passing to deal with multi-agent reinforcement learning.
Finally, we move on to the multi-agent trajectory prediction task. This setting provides is an interesting application of message passing, as the nature of the interactions among the agents may or may not be known a priori. Some recent work are:
Evolve Graph \citep{li2020evolvegraph}, Social Attention \citep{vemula2018social}, SocialGAN \citep{gupta2018social}, and STGAT \citep{huang2019stgat}.

Adversarial imitation learning methods train a discriminator network to estimate the divergence between the expert and the policy's state distribution, and use this discriminator for policy improvement. \citet{ghasemipour2020divergence} provide an unified perspective on this family of algorithms and show that they can be used to minimize any $f$-divergence between the expert and policy state distributions. \citet{fu2017learning} extend this line of work, and present an inverse reinforcement learning algorithm based on this adversarial setup. \citet{ni2020f} move away from this adversarial setup, and present an IRL algorithm that directly optimizes the reward function to minimize any $f$-divergence between the expert and policy state distributions. Another line of work focuses on utilizing meta-learning methods to enable few-shot imitation or reward inference \citep{yu2019meta, xu2019learning}.

\section{Conclusion}
We introduced SS-MAIL to effectively model multi-agent experts and improve the stability of their training dynamics compared to previous AIL methods. Our method comprises of:

\begin{enumerate}
  \item SS-AIL: a self-supervised cost-regularized apprenticeship learning framework that has considerably more stable training dynamics compared to previous AIL methods, while also providing the flexibility for reward-design.
  \item MAIL: A graph convolutional multi-agent actor-critic framework, which aids in multi-modal trajectory generation and alleviates compounding errors.
  \item Trajectory Forcing: a teacher forcing based curriculum that takes advantage of our SS-AIL formulation to stabilize the training dynamics and alleviates the issue of domain shift between training and testing.
\end{enumerate}

% In this work, we highlighted the issue of unstable training dynamics in previous adversarial imitation learning methods. We proposed a self-supervised loss for the discriminator that provides a richer reward for policy improvement and makes the policy training more stable and efficient. We empirically demonstrated the strengths of this algorithm in several multi-agent prediction and control tasks, namely, 1) ability to model multi-modal behavior, 2) stable training, 3) improved sample efficiency, and 4) ability to deal with compounding errors in noisy data. Our experiments show that these strengths transfer to superior empirical performance on real-world datasets, such as MoCap.

The improved stability of the training dynamics and the increased sample efficiency of SS-MAIL allows us to train policies that are robust to weight initializations, which would enable efficient training of multi-agent interactions. Moreover, the ability to handle compounding errors would ensure enhanced generalization during test time, in the presence of noisy inputs. Finally, learning from expert imitation provides a useful framework to train control policies in the absence of the ground truth reward function, and SS-MAIL can be applied to such multi-agent control setting, such as autonomous driving and robotic interactions to learn proficient policies that achieve human-level performance.

\section{Reproducibility}

This section describes our efforts towards ensuring reproducibility of our experiments. First, we provide anonymized code for reproducing all of our experimental results. The code contains documentation for how to generate and plot each experimental result in this paper. In addition, we also provide details about our experimental setup, training hyperparameters, and datasets used in the appendix, which should be sufficient for an independent implementation. All of our experiments are done on at least five seeds to ensure reproducibility. We use author-provided implementations of DNRI and SAC to ensure consistency. Our experiments were done using a single Quadro RTX 6000 GPU, and therefore, are reasonably easy to reproduce within a small computational budget.

\bibliography{iclr2022_conference}
\bibliographystyle{iclr2022_conference}

\appendix
\section{Appendix}
In this section, we provide the derivation and the proof for the Theorem 3.1, to establish the connection between Apprenticeship learning and SS-AIL.
\subsection{Proof}
Ho and Ermon show that the cost-regularized apprenticeship learning problem
\begin{align}
\operatorname{AL}_{\psi}\left(\pi_{E}\right)= \min_\pi \max_c -\psi(c) + E_{\pi} [c(s,a)] - E_{\pi_E}[c(s,a)]
\end{align}
the above is equivalent to min-max problem

\begin{align}
\min_{\pi} \psi^*\left(\rho_{\pi}, \rho_{E}\right)
\end{align}
where $\psi^*$ is the convex conjugate of the cost regularizer $\psi$.

For SS-AIL, 

\begin{align}
\psi_{\mathrm{SS}}^{*}\left(\rho_{\pi}-\rho_{\pi_{E}}\right)= \mathbb{E}_{\pi, \pi_E, \alpha}\left[ 
\underbrace{\left(0-D(s_G,a_G)\right)^2}_{\text{Generated MSE}} + 
\underbrace{\left(1 - D(s_E,a_E)\right)^2}_{\text{Expert MSE}} + \right.
\left. \underbrace{\left(\alpha - D(s_\alpha,a_\alpha)\right)^2}_{\text{Self-Supervised MSE}}  
\right] 
\end{align}

The SS-AIL algorithm aims to solve the min-max optimization problem $\max_\pi \min_D \mathbb{E}_{\pi, \pi_E, \alpha}[ 
\left(0-D(s_G,a_G)\right)^2 + 
\left(1 - D(s_E,a_E)\right)^2 + 
\left(\alpha - D(s_\alpha,a_\alpha)\right)^2 ]$ using the gradient descent-ascent algorithm. The inner discriminator optimization is a supervised learning problem with the self-supervised loss. The outer policy optimization is an off-policy policy improvement problem, and we use the soft actor-critic algorithm.

Therefore, SS-AIL is an instantiation of cost-regularized apprenticeship learning with the cost regularizer $\psi_{SS}$.

\subsection{Further Elaboration on Self-Supervision for Reward Shaping}
We start by elaborating on the self-supervised sampling used to create Fig 1. We observed that sampling $\alpha \in [-1,1]$ was not sufficient for reproducing a drop in the reward function after the expert trajectories. Thus, we increased our sampling range to $[-1, 1.5]$ and set the reward function to 0 for values of $\alpha >1$. This helped our discriminator learn a reward function that does not monotonically increase as we get farther from the current policy. 

We can always split our sampling procedure to get different results. We can make use of the non-linearity of the neural-network architectures to even approximate piece-wise distributions. Another extension could be to have steeper slopes farther away from the current policy and smoother slopes closer to the policy state-distribution. 

\section{Implementation Details}
We begin by delineating all the experimental details that are common to each of the experiments. Our code is built on Pytorch version 1.2, and we use Adam as our optimizer of choice. We follow the teacher forcing training regiment provided in DNRI \citep{graber2020dynamic} to train our baseline reference models for comparison. 

We save the best models after every epoch, and use the models with the best validation loss for testing. Similar to \citet{graber2020dynamic}, we normalize our inputs between -1 to 1 using the maximum and the minimum state norms. For all the experimental evalutaions, we average the results over 5 seeds.
\subsection{Code}
We have provided a driver code in the supplementary material, and the specific hyperparameters in separate folders, one for each figure.
\subsection{Y-Junction Experimental Setup}
\begin{figure}
    \centering
    \includegraphics[width=\linewidth]{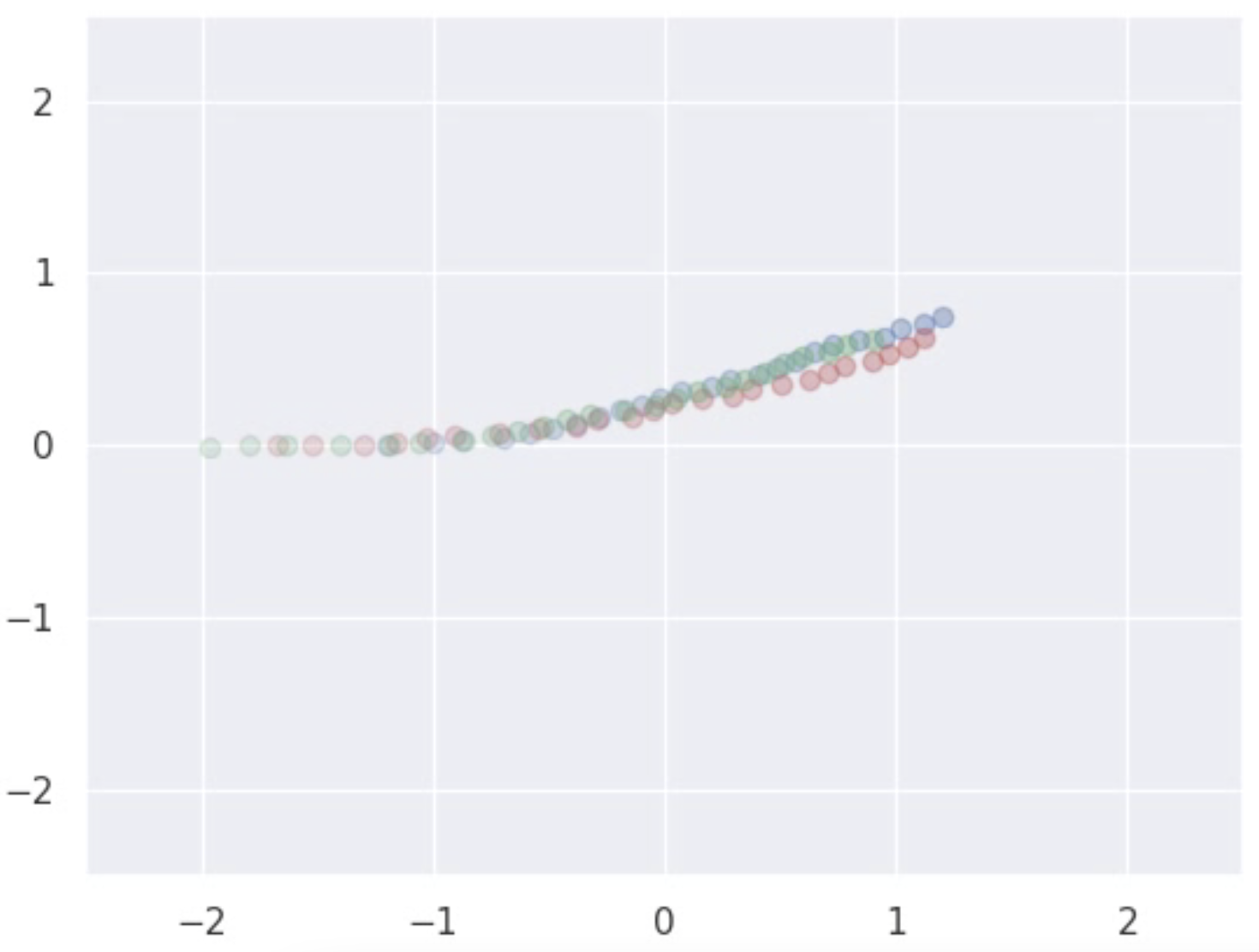}
    \caption{Sample snapshot of the Y-Junction Experiment}
    \label{fig:yjunc}
\end{figure}

We expand upon the simulated synthetic experiments of  \citep{graber2020dynamic}, and use the same number 

\subsection{NOisy Mocap Setup}
For noisy mocap, we add gaussian noise of 0.05 standar deviation.

\end{document}